\documentclass[twoside]{deon16}

\usepackage{deon16macro}

\usepackage{graphicx}
\usepackage{amsmath}
\usepackage{amssymb}

\usepackage{deon16macro}
\usepackage{graphicx}
\usepackage{amsmath}
\usepackage{amssymb}
\usepackage{times}
\usepackage{soul}
\usepackage{url}
\usepackage{xcolor}
\usepackage{amsfonts}
\usepackage{mathrsfs}
\usepackage{tikz}

\usepackage{mathtools}
\DeclarePairedDelimiter\set\{\}
\usetikzlibrary{positioning,fit,arrows.meta,backgrounds}
\usetikzlibrary{shapes,decorations,arrows}
\usetikzlibrary{calc}
\usetikzlibrary{arrows, automata}

\usepackage{xspace}

\newcommand{\Ob}{\mathbf{O}}
\newcommand{\Pe}{\mathbf{P}}
\newcommand{\Li}{\mathbf{L}}

\def\lastname{Calegari, Loreggia, Lorini, Rossi, Sartor}

\begin{document}
\begin{frontmatter}
\title{Modeling Contrary-to-Duty with CP-nets }
\author{Roberta Calegari} 
\address{University of Bologna \\ roberta.calegari@unibo.it}  
\author{Andrea Loreggia}
\address{European University Institute \\ andrea.loreggia@gmail.com} 
\author{Emiliano Lorini}
\address{IRIT \\ lorini@irit.fr}
\author{Francesca Rossi}
\address{IBM Research \\ francesca.Rossi2@ibm.com}
\author{Giovanni Sartor}
\address{European University Institute  \\ giovanni.sartor@gmail.com}

\begin{abstract}
In a ceteris-paribus semantics for deontic logic, a state of affairs where a larger set of prescriptions is respected is preferable to a state of affairs where some of them are violated. Conditional preference nets (CP-nets) are a compact formalism to express and analyse ceteris paribus preferences, which nice computational properties. This paper shows how deontic concepts can be captured through  conditional preference models. A restricted deontic logic will be defined, and   mapped into conditional preference nets.  We shall also show how to model contrary to duties obligations in CP-nets and how to capture in this formalism the distinction between strong and weak permission.
\end{abstract}

\end{frontmatter}
 
\section{Introduction}
Modelling deontic notions through preferences \cite{Hansson2007SV} has the advantage of linking deontic notions to the manifold  research on preferences, in multiple disciplines, such as  philosophy, mathematics, economics and  politics. In recent years, preferences have also been addressed  within  AI \cite{pigozzi2016preferences,domshlak2011preferences,rossi2011short} and applications can be found in multi-agent systems \cite{shoham2008multiagent} and recommender systems \cite{ricci2011introduction}.  

We shall model deontic notions through ceteris-paribus preferences, namely, conditional preferences for a state of affairs over another state of affairs, all the rest being equal. In particular,  we shall focus on the ceteris-paribus preference for a proposition over its complement. The idea of ceteris-paribus preferences was  originally introduced by the philosopher and logician Georg von Wright \cite{von1963logic}. It provides the intuition at the basis  CP-nets \cite{boutilier1999reasoning}, a compact formalism which allows for representing preferences and reasoning about them. 

Though some contributions have linked deontic concepts to ceteris paribus preferences (see \cite{BenthemGrossiLiu2011On}), none has so far focused on CP-nets.  While the intersection between deontic logic and CP-nets is yet unexplored, we think a cooperation among researches from the two fields can be beneficial, and possibly lead to a new workable approach to modelling norms and reason about them. 

We shall indeed assume that norms establishing an obligation can be viewed as expressing a social preference for situations in which  the obligation is complied with over situations in which it is violated. Similarly, we shall assume that norms establishing liberties (bilateral permissions)  express indifference over complementary situations. 

We shall discuss two controversial topics in deontic logic, namely contrary to duty obligations and the distinction between strong and weak permissions.  

We shall show that a consistent CP-net can be built that captures both an obligation and the prescriptions specifying what should to be done in case the obligation is violated, which may be incompatible with prescriptions on what should be done in case of compliance  \cite{Chisholm1957PP}. We shall illustrate contrary to duties obligations with an amusing example discussed in the literature. However, realistic examples of contrary to duty obligations can be found in various domains, such as in commercial contracts, where repair obligations can be established for delay or non-fulfilment: to inform, compensate the damage, replace defective goods, pay penalties, etc. (for a logical model see \cite{GovernatoriRotolo2011JD}. It can also be found in civil liability, where harmful behaviour is met with obligations to repair, mitigate, and pay punitive damages in case of persistence.

We shall also show that a CP-net model can capture concepts of strong and weak permission, a distinction that has been long discussed in deontic logic  \cite{vonWright1963NA} and even earlier in legal theory, in connection issues concerning the completeness of legal systems \cite{Zitelmann1903LR}. We shall indeed  distinguish  the case in which an action is positively permitted, through a permissive norm, and the case in which  the action is rather unregulated, as no norm, and in particular, no prohibition addresses it. 

\section{A running example}\label{example}
In this section, we introduce a running example, concerning the presence of cats, dogs, and fences in beach houses (developing the example from \cite{prakken1997dyadic}). The example is used throughout the paper to explicate new notions when they are introduced.

\begin{example}[Running example] 
\label{ex:marypref}
Mary is the mayor of the Cattown, the city of cat lovers. She knows that her community, with few exceptions,  dislikes dogs and likes cat. Cattown people also dislike the sight of fences around houses, and generally prefer  to be able to move around and visit each other without the obstacle of fences.  However, there was a problem  with dogs entering other people's property and causing fear and sometimes harm. Moreover, people are fuzzy about the way in which their town looks,  they like that all houses and street furniture are white, as indeed they are. Mary believes that her role consists in defining  policies that  fit the  preferences of her constituency. Thus, she has enacted the  following, making them mandatory in Cattown: 
\begin{itemize}
    \item dogs are forbidden
    \item people may have a cat or not
    \item if there is dog, then there  should be a fence
    \item if there is a fence, it should be white.
    \item if there is no dog, then there should be no fence.
\end{itemize}
Though most people in her constituency like these policies, some do not. In particular, there a few dog lovers in the city, who did not share the approach to dogs prevailing in Cattown.  However, Mary believes that --given the preferences of the majority and the history and peculiar spirit of Cattown-- her regulation  articulates a social perspective, namely, a view of the community on what it prefers and  a view that can justifiably be  imposed on all its citizens. If a dog lover has decided to move into Cattown, he should have known what he would find there. 

An issue has recently emerged concerning bobcats. It is being debated whether in the absence of a provision prohibiting or allowing bobcats,  John is allowed to keep this bobcat in his  garden. Mary thinks that he is not, he should have asked for it, while John believes that he is. Can they conflict of opinion be explained away as merely concerning a conceptual misunderstanding (on the notion of a permission)?  
\end{example}

\section{Deontic language and Contrary to duty}\label{deontic}
To formally express rules such as those enacted by Mary, a restricted deontic language is sufficient. Here we provide a specification of this language.
Let $\mathit{Atm}$ be a countable set of atomic propositions and let $\mathit{Lit_{Atm}} = \mathit{Atm} \cup \{ \neg p : p \in \mathit{Atm} \}$, be the corresponding set of literals. Under this assumption, we can identify propositional valuations (or worlds) with maximal consistent conjunctions of literals.

\begin{definition}[A restricted norm language].  The  language $\mathcal{L}_{\mathrm{RDL}}(Atm)$ over a set of atoms $Atm$ is defined as follows:
\begin{itemize}
\item a norm of $Atm$ is any expression having the form $\Ob(\psi|\phi)$ or  $\Pe(\psi|\phi)$  where $\phi$ is a literal in $\mathit{Lit_{Atm}}$ and $\psi$ is a conjunction of such literals
\item $\mathcal{L}_{\mathrm{RDL}}(\mathit{Atm})$ is the set of all norms of $Atm$ and conjunctions of them.
\end{itemize}
\end{definition}
Formulas $\Ob(\psi|\phi)$ and $\Pe(\psi|\phi)$ have to be read, respectively, ``under condition $\psi$, $\phi$ is obligatory" and ``under condition $\psi$, $\phi$ is permitted."
Note the limitation in our language that does not allow for deontic operators $\Ob$ and $\Pe$ to be applied to disjunctions.  

Unconditional obligation and permission do not need to be added as primitives in the language of the logic as they are definable from  conditional obligation and permission. We also do not need a primitive for the bilateral permission, or liberty, which consists in the permission both of a proposition and of its complement.
\begin{definition}\label{noncondcond}
For all $\varphi \in \mathcal{L}_{\mathrm{RDL}}(\mathit{Atm})$:\\
\[
\begin{array}{rcl}
\Ob(\phi) & =_{def} & \Ob(\top|\phi)\\
\Pe(\phi) &=_{def}  & \Pe(\top|\phi)\\
\Li(\psi|\phi) &=_{def}& \Pe(\psi|\phi)\land \Pe(\psi|\neg\phi)\\
\Li(\phi) & =_{def}  & \Li(\top|\phi)\\
\end{array}
\]
\end{definition}

Using this language and taking into account the previous proposition, we can model obligations and permissions depicted in Example \ref{ex:marypref}. Thus we use the following abbreviations $d$ for "there is a dog", $c$ for "there is a cat", $f$ for "there is a fence", $w$ for "the fence is white", $b$ for "there is a bobcat". We can express the norms enacted by Mary as follows:
$$
\begin{array}{l}
    \Ob(\neg d)\\
    \Li(c)\\
    \Ob(\neg d|\neg f)\\
    \Ob(d|f)\\
    \Ob(f|w)
\end{array}
$$

\subsection{Contrary to duty obligations}
The norms in our running example  include two ``contrary to duty'' obligations, namely, obligations that are triggered by the violation of another obligation. According to the norm $\Ob(d|f)$, having a dog, and thus violating  norm $\Ob(\neg d$) (the prohibition to have  dogs), triggers the obligation to have a fence. Similarly, according to the norm  $\Ob(f|w)$, having a fence,  and thus violating  norm  $\Ob(\neg f$) (the prohibition  to have a fence),  triggers the obligation to have it white.

Contrary to duty obligations cannot be captured by standard deontic logic, according to whose semantics a proposition is obligatory if and only if it is true in every perfect (ideal) world, a world in which everything is as it should be. But the world in which a contrary to duty obligation is triggered is subideal one, since in it another obligation is violated, and it remains subideal even if the contrary to duty obligation is complied with. In all perfect  worlds in Cattown there are no dogs, and in such worlds there are also no fences in Cattown. So, according to the semantic for standard deontic logic  there cannot be an obligation to have fences. 
Various attempts have been made to capture the idea of contrary to duty (see for instance \cite{van-der-Torre1999,carmo2002-deontic,governatori2006}).

The problem is that a semantics that just distinguishes between ideal and non-ideal worlds cannot capture contrary to duty obligations: compliance with a contrary to duty obligation takes us to a world that is better -- ceteris paribus -- then a world in which the contrary to duty is not complied with, but which is still imperfect. A world with dogs and fences is ceteris paribus better than a world which does not have fences, but the former is still worse than a world in which, ceteris paribus, there are no dogs. 

To capture contrary to duty obligations we need a semantics that distinguishes different levels of preferability, so that we can distinguish the imperfect situation in which the main duty is violated, but the contrary to duty obligation is complied, from the even more imperfect situation in which both the main duty and the contrary to duty obligation are violated.

\subsection{Strong and weak permissions}
Our running example also shows the difference between strong (explicit) and weak (tacit) permissions, an issue much debated within deontic logic. On the one hand, cats are regulated: there is a norm that deals with cats ($\Li(c)$), by stating that both having and not having cats is allowed. On the other hand, nothing is said about bobcats. In other terms,  while the normative system expresses an equal preference for having and not having cats (both are OK), it express no attitude towards  bobcats. 

This distinction too cannot be captured by the semantics for standard deontic logic, where $\Pe(\phi)$ means that there is at least one perfect world in which $\phi$ is true, in which case  $\neg\Ob(\neg \phi)$ is true. 

In legal logic, the most popular perspective to distinguish strong and weak permissions is an inferential one. To say that $\phi$ is strongly permitted relative to a normative system means that the normative entails $\Pe(\phi)$. To say that $\phi$ is weakly permitted means that the normative system under consideration does not entail a corresponding prohibition $\Ob(\neg\phi)$ .

Here we shall provide for a different way to capture the difference between strong and weak permission, based on the distinction between indifference (equal ceteris-paribus preference) and incomparability (absence of any preference). 

\section{CP-nets}\label{CP-net}
In order to be able to represent the  described scenario, we propose to leverage on existing preference frameworks from AI, namely CP-nets.
A CP-net \cite{boutilier1999reasoning} is a compact representation of conditional preferences in the ceteris paribus semantics. Conditional preferences describe how the preference of an individual for a specific feature can depend on the choice over some other features. A common example is "I prefer to drink red wine if meat is served. Otherwise, if fish is served, I prefer to drink white wine, all the rest being equal”. CP-nets and related extensions of the formalism \cite{CGMR+13a,CGGM+15a} are often on multi-valued variables, this means that variables are not strictly boolean. Nevertheless, in this work we assume that all the variables have binary domains.  We start by introducing and defining some useful notions.

\subsection{Ceteris-Paribus Preferences}
\label{cppref}
We assume a set of variables (also called features or attributes) $V = \{X_1, \ldots ,X_n\}$ each one with binary domain, i.e. the domain of each $X_i \in V$ is such that $Dom(X_i) = \{x_i,\overline{x_i}\}$. An instantiation $Ass(X)$, where $X \subseteq V$, is an assignment of values to $X$. When $X=V$ then we call it a \emph{complete assignment} or \emph{outcome} or \emph{world}, otherwise we call it a \emph{partial assignment}. 
The notation $o[V_i]$ returns the value of the variable $V_i$ in the outcome $o$, while $o[-V_i]$ returns the value of all the variables but $V_i$ in the outcome $o$.

A preorder is a reflexive and transitive binary relation $\preceq$ over the set of all outcomes $O=2^V$, which are usually denoted by $o, w, \ldots, v$. In what follows, $w \preceq v$ means that $v$ is at least as good/ideal as $w$, $w \approx v$ means that $w$ is equivalent to $v$ which is an abbreviation for $w \preceq v$ and $v \preceq w$. Moreover, $w \prec v$ means that $v$ is better/more ideal than $v$ to be an abbreviation of $w \preceq v$ and $v \npreceq w$. Given two possible outcomes $w,v \in O$, if neither $w \preceq v$ nor $v \preceq w$  are valid, then we say that $w$ and $v$ are incomparable, denoted with $w \bowtie v$. 

We borrow some notations from \cite{boutilier1999reasoning} in order to define \emph{ceteris paribus} semantics.

\begin{definition}[Ceteris-Paribus Preference]
\label{def:cppref}
Given a set of variables $V$ and nonempty sets $X,Y,Z$ that partition $V$. Let $z \in Ass(Z)$ and $x_1, x_2 \in Ass(X)$, we say that $x_2$ is \emph{ceteris-paribus at least as good as} or \emph{better than} $x_1$ given $z$ if and only if for all $y_1, y_2 \in Ass(Y)$ 
$$
x_1y_1z \preceq x_2y_2z \text{ and } x_1y_2z \preceq x_2y_1z
$$ 
we use $z: x1 \preceq_Y x2$ ($\prec_Y$ resp.) as an abbreviation for it (we will omit the subscript when it is clear from the context).
\end{definition}

As the preference for  $x_1$ over  $x_2$, in the context of $z$, is independent from the values of Y,  we may say that $X$ is preferentially independent of $Y$ if $Z=\emptyset$ otherwise  that $X$ is said to be conditionally preferentially independent of $Y$ given the assignment of $Z$.

Moreover, from Definition \ref{def:cppref} we derive the following notation regarding ceteris-paribus relation between two outcomes $u,v $. 
\begin{definition}
Given two outcomes $u, v \in U$, we say that $u$ is \emph{ceteris-paribus at least as good as} or \emph{better than} $v$ relative to $X_i$, if and only if they differ only in the value of the variable $X_i$ and 
 we abbreviate it with $v \preceq_{u[X_i]} u$.
\end{definition}


Let us refer to Example \ref{ex:marypref}. Mary's community preferences are over a set of five features, $V= \set{C, D, F, W, B}$, with the following domains: $Dom(C)=\{c,\bar{c}\}$, $Dom(D)=\{d,\bar{d}\}$, $Dom(F)=\{f,\bar{f}\}$, $Dom(W)=\{w,\bar{w}\}$ and $Dom(B)=\{b,\bar{b}\}$. The meaning of variables are the same introduced in Section \ref{deontic} for atoms: for instance, variable $C$ concerns cats, it value is $c$ of $\bar{c}$ depending on whether there is a cat or not.

Since $V$ has 5 elements, the set of all outcomes $U = 2^{V}$ contains 32 possible outcomes. For the sake of readability, we do not consider in this section variable $W$ and $C$ in order to consider a reduced number of outcomes and make the partial orders limited in size. Thus, considering the subset $\{B,D,F\}$, the set of outcomes are denoted by the complete assignments to considered variables:

$$
U=
\{
 bdf,
 \bar{b}df,
  b\bar{d}f,
  bd\bar{f},
 \bar{b}d\bar{f},
 b\bar{d}\bar{f},
 \bar{b}\bar{d}f,
 \bar{b}\bar{d}\bar{f}
\}
$$
The different attitudes of the community can be represented with a set of ceteris paribus preferences. In particular, the negative attitude to dogs is represented by preference for $\neg$dogs-assignments over ceteris-paribus dog-assignments (i.e. $ d \prec_{V \setminus D} \bar{d}$):
$$
\begin{array}{l}
bdf \prec b\bar{d}f, 
\bar{b}df \prec \bar{b}\bar{d}f,
bd\bar{f}, b\bar{d}\bar{f},
\bar{b}d\bar{f}\prec \bar{b}\bar{d}\bar{f}
\end{array}
$$
Preference is ceteris-paribus (about $d$) in the sense that it only concerns the comparison between dog-assignments and $\neg$dog-assignments that are equal in all the rest (bobcats and fences). These are all pair of outcomes $u$ and $v$ such that the evaluations of $u$ and $v$ coincide in all values except for $D$.

Preference for having fences when there are dogs in the house is captured by  preferences for fence-assignments over the ceteris-paribus $\neg$-fence-assignments, taking only dog-assignments into consideration (i.e. $d: \bar{f} \prec_{B} f$):
$$
\begin{array}{l}
bd\bar{f} \prec bdf, \bar{b}d\bar{f} \prec \bar{b}df
\end{array}
$$

Instead, unawareness to bobcats can be modelled through incomparability between bobcat-assignments and $\neg$-bobcat-assignments, which expresses unfamiliarity to the presence of bobcats, i.e. for all $x_1,x_2 \in Ass(D)$ and $y_1, y_2 \in Ass(F): bx_1y_1 \bowtie \bar{b}x_2y_2$.

At last, let us consider to have cats and not to have them, this can be represented through the indifference for cat-assignments over ceteris paribus $\neg$-cat-assignments (i.e. $ c \approx_{V \setminus C} \bar{c}$). These are all pair of outcomes $u$ and $v$ where $u[-C]=v[-C]$ such that $v$ is weakly preferable to $u$ and vice-versa ($u \preceq v \land v \preceq u$), that is for instance:
$$
\begin{array}{l}
cb\bar{d}f \approx \bar{c}b\bar{d}f, cbd\bar{f} \approx \bar{c}bd\bar{f}
\end{array}
$$


\subsection{Incomparability}

In the previous section, we introduced the notion of incompatibility. In the ceteris paribus semantics, the notion of incomparability does not allow to differentiate between two different states of affairs: one where two outcomes are incomparable because nothing is said about some values of a  feature, and another, where two or more variables are independent and thus preferences are described over the domains of these variables no matter the assignment of the others. Both cases induce a preference graph where some outcomes are incomparable because there does not exist a path between them, i.e. given two incomparable outcomes $w,v$ then $w \npreceq^T v$, $v \npreceq^T w$,$w \nprec^T v$ and $v \nprec^T w$. 
But, while the former case describes a lack of information, possibly because the individual does not know or does not have any information about some features (for instance in our example the lack of information about bobcats), the latter describes a situation where preferences are reported on all the features but not on the combination of them. The lack of dependencies seems to be simpler to fix: the individual already has all the information on the features, thus the incomparability can be avoid by adding a dependency between some of the variables. 

\begin{definition}[Strong incomparability]
Given a preorder $P$, two outcomes $w,v \in P$ are said to be \emph{strongly incomparable} if they belong to two different  components of $P$.
\end{definition}

Strong incomparability describes the relation among outcomes of different components; they are incomparable because some information is missing, specifically preferences are not reported for some features. This lack of information can be hard to fix, since it entails a lack of knowledge.

The following definition refers to \emph{weakly connected} components: a directed graph is said to be weakly connected if the undirected graph resulting from removing the orientation of the edges is connected, i.e. if any pair of vertexes $w,v$ in the undirected graph has a path from $w$ to $v$.

\begin{definition}[Weak incomparability]
Given a preorder $P$, two outcomes $w,v \in P$ are said to be \emph{weakly incomparable} if they belong to the same weakly connected component of $P$ but there does not exist a path from $w$ to $v$ or vice-versa.
\end{definition}

Weak incomparability describes a less problematic situation. In this case, preferences are reported but     the incomparability is due to some missing dependencies among features. Thus, correcting this kind of incomparability can be quite simple.

\subsection{Conditional Preference Networks}

Given the above definitions, we can now define conditional preference networks (CP-net).

\begin{definition}
A CP-net over a set of binary variables $V=\{V_1, \ldots,V_n\}$ is a tuple $\mathcal{N}=(G,CPT)$, where $G=(V,E)$ is a directed graph and $CPT=\{CPT(V_i) | V_i \in V\}$ is a set of conditional preference tables (or CP-tables). An edge $(V_i, V_j)\in E$ represents that preferences over $Dom(V_j)$ depend on the value of $V_i$.
\end{definition}

For each variable $V_i \in V,$ a  $CPT(V_i)$ is a set of cp-statements, each one represents an ordering over the specific domain $Dom(V_i)$ given the assignment to the parents of $V_i$, e.g. $CPT(D)=\{\bar{d} \prec d\}$. In its original formulation, each $CPT(V_i)$ reports strict linear order over the values of the domain of $V_i$ given the partial assignment to $Pa(V_i)$. A more recent extension, namely CP-net with indifference \cite{allen2013cp}, takes into account indifference and it also models lack of information through incomparability. In this work we adopt this extension of CP-net. 
The semantics is connected with the induced preference graph over all the outcomes, that is a complete assignment of values to variables. A directed edge between pair of outcomes $(o_i, o_j)$, which differ only in the value of one variable, means that $o_j \preceq o_i$. A \emph{worsening flip} is a change in the value of a variable to a less preferred value according to the cp-statement for that variable. 

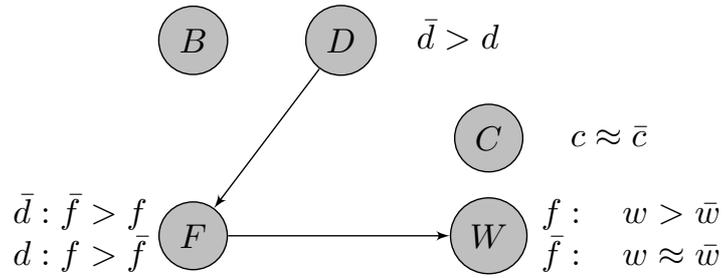
\begin{figure}
\centering
\resizebox{0.9\linewidth}{!}{
\begin{tikzpicture}
\tikzstyle{every node}=[draw,shape=circle,fill=gray!50];
\tikzset{edge/.style = {->,> = latex'}}
\node (B) at (0, 1) {$B$};
\node (C) at (3, 0) {$C$};
\node (D) at (1.5, 1) {$D$};
\node (F) at (0, -1) {$F$};
\node (W) at (3, -1) {$W$};
\draw[edge]  (D) to (F);
\draw[edge]  (F) to (W);

\coordinate [label=left: {
\begin{tabular}{c}
$ \bar{d} > d$ \\
\end{tabular}
}] (p) at (3.5,1);

\coordinate [label=left: {
\begin{tabular}{c}
$ c \approx \bar{c}$ \\
\end{tabular}
}] (p) at (5,0);

\coordinate [label=left: {
\begin{tabular}{c}
$ \bar{d}: \bar{f} > f$ \\
$ d: f > \bar{f}$ \\
\end{tabular}
}] (p) at (0,-1);
\coordinate [label=right: {
\begin{tabular}{cc}
$ f: $	& $ w > \bar{w}$ \\
$ \bar{f}: $ &   $ w \approx \bar{w}$ \\
\end{tabular}
}] (p) at (3,-1);
\end{tikzpicture}
}
\caption{The CP-net with indifference which represents Mary's community preferences of Example \ref{ex:marypref}.}
\label{figure:excpnet1}
\end{figure}
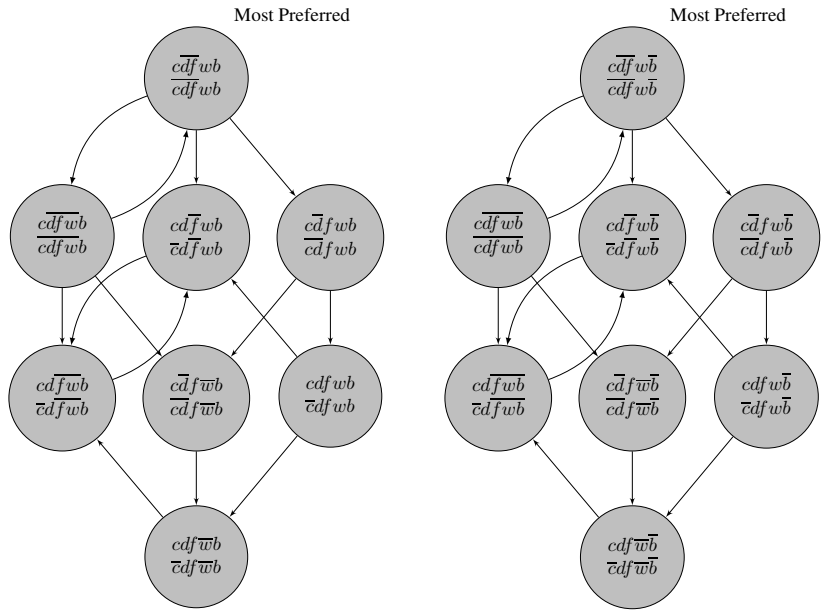
\begin{figure}
\centering
\resizebox{0.9\linewidth}{!}{
\begin{tikzpicture}
\tikzstyle{every node}=[draw,shape=circle,fill=gray!50];
\tikzset{edge/.style = {->,> = latex'}}
\node[label={[label distance=0.1]20:Most Preferred}] (a) at (-4, 0) {\begin{tabular}{c}
$c\overline{df}wb$ \\
$\overline{cdf}wb$ \\
\end{tabular}
};

\node[below=of a, xshift=-70] (b) 
{\begin{tabular}{c}
$c\overline{dfw}b$ \\
$\overline{cdfw}b$ \\
\end{tabular}
};

\node[below=of a, xshift=70] (c) 
{\begin{tabular}{c}
$c\overline{d}fwb$ \\
$\overline{cd}fwb$ \\
\end{tabular}
};

\node[below=of a] (d) 
{\begin{tabular}{c}
$cd\overline{f}wb$ \\
$\overline{c}d\overline{f}wb$ \\
\end{tabular}
};

\node[below=of d, xshift=-70] (e) 
{\begin{tabular}{c}
$cd\overline{fw}b$ \\
$\overline{c}d\overline{fw}b$ \\
\end{tabular}
};

\node[below=of d] (f) 
{\begin{tabular}{c}
$c\overline{d}f\overline{w}b$ \\
$\overline{cd}f\overline{w}b$ \\
\end{tabular}
};

\node[below=of d, xshift=70] (g) 
{\begin{tabular}{c}
$cdfwb$ \\
$\overline{c}dfwb$ \\
\end{tabular}
};

\node[below=of f] (h) 
{\begin{tabular}{c}
$cdf\overline{w}b$ \\
$\overline{c}df\overline{w}b$ \\
\end{tabular}
};

\draw[-latex, bend right]   (a) edge [bend right] (b);
\draw[-latex, bend right]   (b) edge [bend right] (a);
\draw[edge]   (a) -> (c);
\draw[edge]   (a) -> (d);
\draw[edge]   (b) -> (e);
\draw[edge]   (b) -> (f);
\draw[edge]   (c) -> (f);
\draw[edge]   (c) -> (g);
\draw[-latex, bend right]   (d) edge [bend right] (e);
\draw[-latex, bend right]   (e) edge [bend right] (d);
\draw[edge]   (g) -> (d);
\draw[edge]   (h) -> (e);
\draw[edge]   (f) -> (h);
\draw[edge]   (g) -> (h);

\node[label={[label distance=0.1]20:Most Preferred}] (a1) at (4, 0) 
{\begin{tabular}{c}
$c\overline{df}w\overline{b}$ \\
$\overline{cdf}w\overline{b}$ \\
\end{tabular}
};

\node[below=of a1, xshift=-70] (b1) 
{\begin{tabular}{c}
$c\overline{dfw}\overline{b}$ \\
$\overline{cdfw}\overline{b}$ \\
\end{tabular}
};

\node[below=of a1, xshift=70] (c1) 
{\begin{tabular}{c}
$c\overline{d}fw\overline{b}$ \\
$\overline{cd}fw\overline{b}$ \\
\end{tabular}
};

\node[below=of a1] (d1) 
{\begin{tabular}{c}
$cd\overline{f}w\overline{b}$ \\
$\overline{c}d\overline{f}w\overline{b}$ \\
\end{tabular}
};

\node[below=of d1, xshift=-70] (e1) 
{\begin{tabular}{c}
$cd\overline{fw}\overline{b}$ \\
$\overline{c}d\overline{fw}\overline{b}$ \\
\end{tabular}
};

\node[below=of d1] (f1) 
{\begin{tabular}{c}
$c\overline{d}f\overline{w}\overline{b}$ \\
$\overline{cd}f\overline{w}\overline{b}$ \\
\end{tabular}
};

\node[below=of d1, xshift=70] (g1) 
{\begin{tabular}{c}
$cdfw\overline{b}$ \\
$\overline{c}dfw\overline{b}$ \\
\end{tabular}
};

\node[below=of f1] (h1) 
{\begin{tabular}{c}
$cdf\overline{w}\overline{b}$ \\
$\overline{c}df\overline{w}\overline{b}$ \\
\end{tabular}
};

\draw[-latex, bend right]   (a1) edge [bend right] (b1);
\draw[-latex, bend right]   (b1) edge [bend right] (a1);
\draw[edge]   (a1) -> (c1);
\draw[edge]   (a1) -> (d1);
\draw[edge]   (b1) -> (e1);
\draw[edge]   (b1) -> (f1);
\draw[edge]   (c1) -> (f1);
\draw[edge]   (c1) -> (g1);
\draw[-latex, bend right]   (d1) edge [bend right] (e1);
\draw[-latex, bend right]   (e1) edge [bend right] (d1);
\draw[edge]   (g1) -> (d1);
\draw[edge]   (h1) -> (e1);
\draw[edge]   (f1) -> (h1);
\draw[edge]   (g1) -> (h1);

\end{tikzpicture}
}
\caption{The partial order induced by the CP-net: for the sake of readability, we group into the same nodes some outcomes of the preference model. Outcomes in the same node are indifferent, this is due to the indifference over the values of variable $C$.}
\label{figure:excpnet2}
\end{figure}

The semantics of the CP-net with indifference is a preorder over all the outcomes, i.e. a reflexive and transitive binary relation over the set of complete assignment of values to the variables of the CP-net. Thus, for any two outcomes $o, u$ which differs only on the value of one variable $X_i \in V$ we have:
\begin{itemize}
    \item $o \preceq u$ if $o[X_i] \preceq u[X_i]$ given $o[Pa(X_i)]$
    \item $o \bowtie u$ if there is not a cp-statements for $o[X_i], u[X_i]$ given $o[Pa(X_i)]$
\end{itemize}

The partial order induced by the CP-net is denoted as $Ord_{\mathcal{N}}$.

The CP-net formalism provides a qualitative compact representation that is useful to represent scenarios similar to the one depicted by the reduced deontic logic introduced in Section \ref{deontic}. 

Based on preferences described in Example \ref{ex:marypref}, Figure \ref{figure:excpnet1} reports the dependency graph and the CP-tables of the CP-net. The CP-net is over the set of variables $V=\{B,C,D,F,W\}$. Variable $B$ does not have a CP-table because the community did not express any preferences about bobcats. Notice the indifference on the values of variable $C$ which describes the liberty to  have or not to have cats. The strong attitude about dogs induces the strict order over the domain of the variable $D$. Orders over the variable $F$ depends on whether or not there is a dog and orders over the variable $W$ depends on whether or not there is a fence.

The partial order by the CP-net is reported in Figure \ref{figure:excpnet2}. The binary relation among outcomes is based on dependencies and preferences reported in the CP-tables of the CP-net. For instance, outcomes $c\overline{df}w\overline{b}$ and $\overline{cdf}w\overline{b}$ are in the same node due to the indifference on the values of the variable $C$ all the rest being equal, while $cd\overline{f}w\overline{b} \prec \overline{cdf}w\overline{b}$ because $\neg$dog-assignments are ceteris-paribus better than dog-assignments.

Notice that the preference graph has two components. Outcomes in the two components differs on the assignment of the variable $B$ for which no preferences are expressed, thus each component represent different scenario which differs from the other for the presence of bobcats. As we described in the previous section, this is represented with strong incomparability between outcomes of different components.

\section{Deontic Language and CP-nets: bridging the gap}\label{mapping}
A set of obligations and liberties expressed in the restricted language define in Section \ref{deontic} can be represented using a CP-net, we will call it a prescriptive CP-net. Note that we do not provide for the representation of unilateral permission. In fact we assume that unilateral permissions must  be implied either by an obligation (in case the complement is forbidden) or by a liberty (in case the complement is also permitted). 

Indeed, the restriction on obligations and liberties where antecedents are conjunctions of literals and consequent is a single literal is compatible with the syntax of CP-net with indifference \cite{allen2013cp}.
\subsection{Modeling Deontic Language with CP-nets}
In this section we define the CP-net that is induced by a given set of prescriptions.

\begin{definition}
Given a set of statements in $C \subseteq \mathcal{L}_{\mathrm{RDL}}(Atm)$ and the prescriptive CP-net $\mathcal{N_C}$. In $\mathcal{N_C}=(G,P)$, $G$ is a directed graph over a set of variables $V$ such that:
\begin{itemize}
    \item for each $v_i \in Atm$ corresponds a $V_i \in V$ with $Dom(V_i)=\{v_i, \bar{v_i}\}$
    \item each conditional obligation $\Ob(\psi|\phi) \in C$ and each liberty $\Li(\psi|\phi) \in C$, introduce dependencies in $G$. Specifically, for each literal $x_j \in \psi$, they introduce a directed edge between  $X_j$ and $\Phi$ in $G$, thus $X_j$ becomes a parent of $\Phi$
    \item each obligation $\Ob(\psi|\phi) \in C$ induces a strict order over $Dom(\Phi)$ given the assignment $\psi$ such that $CPT(\Phi)=\{\psi: \bar{\phi} \prec \phi\}$
    \item each conditional liberty $\Li(\psi|\phi) \in C$ induces a weak order over $Dom(\Phi)$ given the assignment $\psi$ such that $CPT(\Phi)=\{\psi: \bar{\phi} \approx \phi\}$. 
\end{itemize}
\label{prop:cpnet}
\end{definition}

Notice that variables with empty CP-tables or partially empty CP-tables may exist. These variables will induce incomparability in the preorder.

\subsection{Mapping Preference Models with CP-nets}
In this section we shall provide a ceteris paribus semantics for the deontic language provided above. In the next session, we shall show that this semantics identifies models that correspond to CP-nets. The semantics of the deontic language is defined in terms of preference relations on worlds.

\begin{definition}[Preference models of $\mathcal{L}_{\mathrm{RDL}}$] A preference model of $\mathcal{L}_{\mathrm{RDL}}, M=(\preceq, U)$ is a preorder $\preceq$ on the set U of outcomes.
\end{definition}

A set of ceteris-paribus conditional obligations and liberties $C$ is \emph{consistent} if it has at least one model, and inconsistent otherwise. A set of ceteris-paribus conditional obligations and liberties $C$ entails another preference formula $C'$, written $C \models C'$, if every model of $C$ is also a model of $C'$.
In the following definition, an outcome refers to conjunctions of propositional literals referring to each variable in $V$ exactly once, and which are consistent. Thus, for instance $\phi$ refers to valuation of variable $\Phi$.
Satisfaction for formulas in the language $\mathcal{L}_{\mathrm{RDL} }(\mathit{Atm} )$ is defined as follows,
where $u \models \varphi$
means that the propositional formula $\varphi$
is true at the valuation corresponding to the outcome 
$u$:
\begin{definition}[Satisfaction in $\mathcal{L}_{\mathrm{RDL}}$]\label{truthcond}
Given a set of norms $C \subseteq \mathcal{L}_{\mathrm{RDL}}$ and a preference model $M = ( U, \preceq )$, we say that $M$ satisfies $C$, if for each $\Ob(\psi|\phi), \Pe(\psi|\phi ) \in C$ the following hold:
\begin{eqnarray*}
M \models \Ob(\psi|\phi) & \Longleftrightarrow & \forall
    v,u \in U:
    v, u \models \psi,   \text{ if }
    v \models \phi \text{ and }\\
&&  v \preceq_{\phi} u
    \text{ then }  u \models \phi\\
M \models \Pe(\psi|\phi ) & \Longleftrightarrow & \forall
    v,u \in U: 
    v , u\models \psi,  \text{ if }
    v \models \phi \text{ and }\\
&&  v \prec_{\phi} u
    \text{ then }  u \models \phi
\end{eqnarray*}

The class of models satisfying a set of norms $C$ is denoted with $\mathcal{P}_C$. Among all the possible $M \in \mathcal{P}_C, rel(\mathcal{P}_C) \subseteq \mathcal{P}_C$ is the subset of models that refers only to consistent conjunctions of propositional literals in $C$.
Then $M_C = (U, \preceq_C)$ is the preference model such that for each 
$M= (U , \preceq) \in rel(\mathcal{P}_C), \preceq_C \subseteq \preceq$.
\end{definition}

\begin{proposition}
As a liberty is a bilateral permission ($\Li(\psi|\phi)=_{def} \Pe(\psi|\phi)\land \Pe(\psi|\neg\phi)$), we have that 
\begin{eqnarray*}
M \models \Li(\psi|\phi ) & \Longleftrightarrow & \forall
    v,u \in U: 
    v \models \psi, u \models \psi,\\
&&      \text{ if }
    v \models \phi \text{ and }
v \prec_{\phi} u
    \text{ then }  u \models \phi,  \text{ and }\\
&&      \text{ if }
    v \models \neg\phi \text{ and }
v \prec_{\neg\phi} u
    \text{ then }  u \models \neg\phi\\
\end{eqnarray*}
\end{proposition}

Given the previous definitions we now show that a set of ceteris-paribus norms in $\mathcal{L}_{\mathrm{RDL}}(Atm)$ is expressible with a prescriptive CP-net, this means that the preorder induced by the CP-net is a model which satisfies $C$.

\begin{theorem}
Given a set of norms $C \subseteq \mathcal{L}_{\mathrm{RDL}}(Atm)$  and the prescriptive CP-net $\mathcal{N}_{C}$, then $Ord_{\mathcal{N}_{C}} = M_C$.
\end{theorem}
\begin{proof}
This is equivalent to prove that for each norm $N \in C$
\begin{equation}\label{eq:equation1}
 M_C \models N 
 \text{ iff } Ord_{\mathcal{N}_{C}} \models N 
\end{equation}
We show how to prove (\ref{eq:equation1}) showing how this can be done for obligations. Let start by assuming that $M_C \models \Ob(\psi|\phi)$, by definition this is equivalent to 
$\forall u,v \in U: u \prec_{\psi,\phi} v$, thus for for consistency of $\mathcal{N_C}$ with $C$:
\begin{eqnarray*}
Ord_{\mathcal{N}_{C}} & \models & \psi: \neg\phi \prec \phi\\
Ord_{\mathcal{N}_{C}} & \models & \forall y_1, y_2 \in (V-\Psi-\{\Phi\}): \neg\phi\psi y_1 \prec \phi\psi y_2\\
\end{eqnarray*}
by definition, this is equivalent to
\begin{eqnarray*}
Ord_{\mathcal{N}_{C}} & \models & \forall u,v \in U: u,v \models \psi, u \models \neg\phi, v \models \phi\\
Ord_{\mathcal{N}_{C}} & \models & \Ob(\psi|\phi)
\end{eqnarray*}
This proves $\implies$ direction. We show the other way round ($\Longleftarrow$) by contradiction. Let us assume that $Ord_{\mathcal{N}_{C}} \nvDash \Ob(\psi|\phi)$, thus $\Ob(\psi|\phi) \notin \mathcal{N_C}$ but this means that also $\Ob(\psi|\phi) \notin C$. But this is a contradiction.
\end{proof}

On the basis of the CP-net modellin described above the following propsitions hold.
\begin{proposition}\label{propframe} A contrary to duty framework $\Ob(\phi), \Ob(\neg\phi|\psi),\Ob(\phi|\neg \psi)$ is modelled by a consistent CP-net.
\end{proposition}
\begin{proof}
The CP-net is reported in Figure \ref{figure:excpnet3}.

\begin{figure}[htb]
\centering
\begin{tikzpicture}
\tikzstyle{every node}=[draw,shape=circle,fill=gray!50];
\tikzset{edge/.style = {->,> = latex'}}
\node (A) at (0, 1) {$\Phi$};
\node (B) at (0, -1) {$\Psi$};

\draw[edge]  (A) to (B);

\coordinate [label=left: {
\begin{tabular}{c}
$ \phi > \bar{\phi}$ \\
\end{tabular}
}] (p) at (0,1);

\coordinate [label=left: {
\begin{tabular}{c}
$\phi: \bar{\psi} > \psi$ \\
$\bar{\phi}: \psi > \bar{\psi}$ \\
\end{tabular}
}] (p) at (0,-1);
\end{tikzpicture}
\caption{The prescriptive CP-net induced by the duty framework of Proposition \ref{propframe}.}
\end{figure}
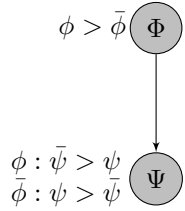

\begin{figure}
\centering
\begin{tikzpicture}
\tikzstyle{every node}=[draw,shape=circle,fill=gray!50];
\tikzset{edge/.style = {->,> = latex'}}
\node[label={[label distance=0.1]20:Most Preferred}] (a) 
{\begin{tabular}{c}
$\phi\bar{\psi}$ \\
\end{tabular}
};

\node[below=of a, xshift=-70] (b) 
{\begin{tabular}{c}
$\phi\psi$ \\
\end{tabular}
};

\node[below=of a, xshift=70] (c) 
{\begin{tabular}{c}
$\bar{\phi}\bar{\psi}$ \\
\end{tabular}
};

\node[below=of c, xshift=-70] (d) 
{\begin{tabular}{c}
$\bar{\phi}\psi$ \\
\end{tabular}
};

\draw[edge]  (a) to (b);
\draw[edge]  (b) to (d);
\draw[edge]  (a) to (c);
\draw[edge]  (d) to (c);

\end{tikzpicture}
\caption{The partial order induced by the prescriptive CP-net which satisfies the duty framework of Proposition \ref{propframe}.}
\label{figure:excpnet3}
\end{figure}
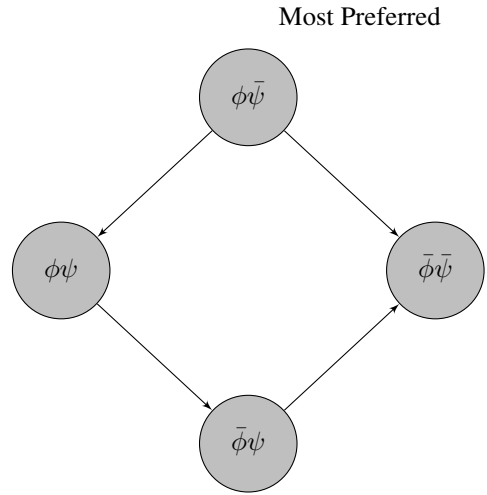
\end{proof}


\section{Conclusions}\label{conclusion}
Our analysis shows that a set of norms expressed in deontic logic can be translated into a corresponding CP-net without loss of semantics, thus providing a conditional preference network giving a compact representation.

On the contrary, CP-nets allow to model differ scenarios and to infer interesting properties and information about the set of norms. The main problems that CP-nets allow to tackle are reasoning about the dominance testing and consistency testing. Both of them are in general far from easy. In the first problem we want to decide whether one outcome $v$ dominates $u$, i.e. $u \prec v$. In the  second problem we want to find out whether there is a dominance cycle in the preorder defined by a CP-net, i.e. whether there is an outcome that dominates (is preferred to) itself. In general both dominance and consistency for CP-nets can be NP-complete \cite{boutilier1999reasoning,goldsmith2008computational}. 
But when the CP-net is over a set of binary variables, and its dependency graph is a tree, the dominance testing becomes linear in the number of variables, if it is a poly-tree than dominance testing becomes polynomial in the number of variables \cite{boutilier1999reasoning}. Moreover, in many scenarios, the presence of cycles in the dependency graph means that the network is inconsistent \cite{domshlak2002cp}. Thus, the consistency testing would allow to determine when a set of norms is inconsistent because contradictory while the dominance test allows to test whether a world is better than another.
Moreover, in recent studies, preferences represented with compact representation are used in metric spaces in order to define how similar they are \cite{li2018efficient,loreggia2018distance}. Furthermore, as preference orders are exponential in the number of considered variables, it is of main importance being able to find feasible ways to compute the distance among preferences exploiting the compact representation of such domains.

The results presented here represent just a preliminary exploration of intersection between CP-nets and deontic logic, but it can provide a starting point for further research.

\bibliographystyle{deon16}
\bibliography{ijcai20}

\end{document}